\documentclass[paper=letter, fontsize=20pt]{article}
\usepackage[square]{natbib}
\usepackage[english]{babel} 
\usepackage{amsmath,amsfonts,amsthm} 
\usepackage[utf8]{inputenc}
\usepackage{float}
\usepackage{lipsum} 
\usepackage{multirow}
\usepackage{array}
\usepackage{blindtext}
\usepackage{graphicx}
\usepackage[labelformat=simple]{subcaption}

\usepackage{caption}
\usepackage{appendix}
\makeatletter
\def\@seccntformat#1{\@ifundefined{#1@cntformat}%
   {\csname the#1\endcsname\quad}  
   {\csname #1@cntformat\endcsname}
}
\let\oldappendix\appendix 
\renewcommand\appendix{%
    \oldappendix
    \newcommand{\section@cntformat}{\appendixname~\thesection\quad}
}
\makeatother

\usepackage[sc]{mathpazo} 
\usepackage[T1]{fontenc} 
\usepackage{microtype} 
\usepackage[hmarginratio=1:1,top=32mm,columnsep=20pt]{geometry} 
\usepackage{multicol} 
\usepackage{booktabs} 
\usepackage{float} 
\usepackage{hyperref} 
\usepackage{lettrine} 
\usepackage{paralist} 
\usepackage{abstract} 
\usepackage{titlesec} 
\usepackage{enumitem}
\usepackage{algorithm}
\usepackage{algorithmic}
\renewcommand{\algorithmicrequire}{ \textbf{Input:}} 
\renewcommand{\algorithmicensure}{ \textbf{Output:}} 

\usepackage{amsfonts}    
\usepackage{amsmath}
\usepackage{amsthm}
\usepackage{dsfont}
\usepackage{mathtools}
\usepackage{xcolor}
\usepackage{bm}
\usepackage{algorithm}
\usepackage{algorithmic}
\allowdisplaybreaks
\usepackage{graphicx}
\usepackage{caption}

\usepackage{url}

\usepackage{multirow}
\usepackage{booktabs}
\usepackage{tabularx}

\usepackage{paralist}

\usepackage{thmtools}
 
\declaretheoremstyle[%
  headfont=\normalfont\itshape,%
  postheadspace=1em,%
  qed=\qedsymbol%
]{mystyle} 
\declaretheorem[name={Proof},style=mystyle,unnumbered,
]{spacelessprf}

\newcommand{\bftab}{\fontseries{b}\selectfont}

\newcommand*\diff{\mathop{}\!\mathrm{d}}
\newcommand*\argmax{\mathop{\arg \max}}
\newcommand*\argmin{\mathop{\arg \min}}

\usepackage{paralist}
\usepackage{bbm}
\renewcommand{\H}{\mathcal{H}} 
\newcommand{\X}{\mathcal{X}} 
\newcommand{\Y}{\mathcal{Y}} 
\newcommand{\E}{\mathbb{E}} 
\renewcommand{\d}{\mathrm{d}} 
\newcommand{\dP}{\mathrm{d} P}

\newtheorem{Theorem}{Theorem}
\newtheorem{Lemma}{Lemma}

\newtheorem{Proposition}{Proposition}

\title{\vspace{-7mm}\fontsize{20pt}{10pt}\selectfont\textbf{Model Reuse with Reduced \\ \vspace{1.2mm} Kernel Mean Embedding Specification}} 
\author{Xi-Zhu Wu}
\author{
Xi-Zhu Wu\textsuperscript{1}, Wenkai Xu\textsuperscript{2}, Song Liu\textsuperscript{3,4}, Zhi-Hua Zhou\textsuperscript{1}\thanks{Corresponding author. Email: zhouzh@nju.edu.cn}\\
\small
\textsuperscript{1}National Key Laboratory for Novel Software Technology, Nanjing University, China.\\
\small
\textsuperscript{2}Gatsby Unit of Computational Neuroscience, University College London, UK.\\
\small
\textsuperscript{3}School of Mathematics, University of Bristol, UK. $^4$The Alan Turing Institute, UK.
}
\date{}

\begin{document}
\maketitle

\hrule
\begin{abstract} 
Given a publicly available pool of machine learning models constructed for various tasks, when a user plans to build a model for her own machine learning application, is it possible to build upon models in the pool such that the previous efforts on these existing models can be reused rather than starting from scratch? Here, a grand challenge is how to find models that are helpful for the current application, without accessing the raw training data for the models in the pool. In this paper, we present a two-phase framework. In the upload phase, when a model is uploading into the pool, we construct a reduced kernel mean embedding (RKME) as a \emph{specification} for the model. Then in the deployment phase, the relatedness of the current task and pre-trained models will be measured based on the value of the RKME specification. Theoretical results and extensive experiments validate the effectiveness of our approach.
\end{abstract} 
\hrule
\section{Introduction}
Recently, the learnware~\citep{learnware} paradigm has been proposed. A learnware is a well-performed pre-trained machine learning model with a specification which explains the purpose and/or specialty of the model. The provider of a learnware can upload it to a market, and ideally, the market will be a pool of (model, specification) pairs solving different tasks. When a person is going to tackle her own learning task, she can identify a good or some useful learnwares from that market whose specifications match her requirements and apply them to her own problem.

One of the most important properties of learnware is to enable future users to build their own applications upon previous models without accessing the raw data used to train these models, and thus, the machine learning experience is shared but the data privacy violation and data improper disclosure issues are avoided. This property is named \emph{inaccessibility} of training data.

Note that it may be too optimistic to expect that there is a model in the pool which was trained exactly for the current task; maybe there is one, or multiple, or even none helpful models. Thus, a key challenge is how to provide each model with a specification such that given a new learning task, it is possible to identify helpful models from the model pool. This property is named \emph{reusability} of pre-trained models.

It was thought that logical clauses or some simple statistics may be used to construct the model specification~\citep{learnware}, though there has been no effective approach yet. In this paper, we show that it is possible to construct a reduced kernel mean embedding (RKME) specification for this purpose, where both inaccessibility and reusability are satisfied under reasonable assumptions.

Kernel mean embedding~\citep{KME_survey} is a powerful technique for solving distribution-related problems, and has made widespread contribution in statistics and machine learning, like two-sample testing~\citep{Interpretable_TSS}, casual discovery~\citep{KME_casual}, and anomaly detection~\citep{KME_anomaly}. Roughly speaking, KME maps a probability distribution to a point i{}n reproducing kernel Hilbert space (RKHS), and can be regarded as a representation of distribution. Reduced set construction~\citep{TNN99_RS} keeps the representation power of empirical KMEs, and blocks access to raw data points at the same time.

To clearly show why reduced KME is a valid specification in the learnware paradigm, we decompose the paradigm into a two-phase framework. Initially, in the upload phase, the pre-trained model provider is required to construct a reduced set of empirical KME as her model's specification and upload it together with the built predictive model into a public pool. The RKME represents the distribution of model's training data, without using any raw examples. Subsequently, in the deployment phase, we demonstrate that the user can select suitable pre-trained model(s) from the pool to predict her current task by utilizing the specifications and her unlabeled testing points in a systematic way.

RKME specification is a bridge between the current task and solved tasks upon which the pre-trained models are built. We formalize two possible relationships between the current and the solved tasks. The first one is \emph{task-recurrent assumption}, saying the data distribution of the current task matches one of the solved tasks. We then use the maximum mean discrepancy (MMD) criteria to find the unique fittest model in the pool to handle all testing points. The second one is \emph{instance-recurrent assumption}, saying the distribution of the current task is a mixture of solved tasks. Our algorithm estimates the mixture weight, uses the weight to generate auxiliary data mimicking the current distribution, learns a selector on these data, then uses the selector to choose the fittest pre-trained model for each testing point. Kernel herding~\citep{herding_thesis}, a fast sampling method for KME, is applied in mimic set generation.

Our main contributions are:
\begin{compactitem}
\item Propose using RKME as the specification under the learnware paradigm, and implement a two-phase framework to support the usage.
\item Show the inaccessibility of training data in the upload phase, i.e., no raw example is exposed after constructing specifications. 
\item Prove the reusability of pre-trained models in the deployment phase, i.e., the current task can be handled with identified existing model(s).
\item Evaluate our proposal through extensive experiments including a real industrial project.
\end{compactitem}

In the following sections, we first present  necessary background knowledge, then introduce our proposed framework, followed by theoretical analysis, related work, experiments and finally the conclusion.

\section{Background} \label{sec:background}
In this section, we briefly introduce several concepts and techniques. They will be incorporated and further explained in detail through this paper.

\subsection{Kernel Mean Embeddings}
Let $X \in \X$ be a random variable in $\mathbb{R}^d$ and $P_X$ be a measurable probability function of $X$. Let $k:\mathcal{X}\times\mathcal{X}\rightarrow\mathbb{R}$ be reproducing kernels for $\X$, with associated RKHS $\mathcal{H}_k$ and $\phi:x\in\mathcal{X}\mapsto k(x,\cdot)\in\mathcal{H}_k$,
the corresponding canonical feature maps on $\X$. Throughout this paper, we assume the kernel function is continuous, bounded, and positive-definite. The kernel function is considered a similarity measure on a pair of points in $\mathcal{X}$.

Kernel mean embedding (KME)~\citep{Smola07KME} is defined by the mean of a $\H_k$-valued random variable that maps the probability distributions to an element in RKHS associated with kernel $k$~\citep{Learning_with_Kernels}. Denote the distribution of an $\mathcal{X}$-valued random variable by $P$, then its kernel mean embedding is
\begin{equation}\label{eq:KME_def}
    \mu_k(P)\coloneqq \int_\mathcal{X} k(x,\cdot)\diff P(x)
\end{equation}
By the reproducing property, $\forall f\in \H_{k}, \langle f, \mu_k(P)\rangle = \E_{P}[f(X)]$, which demonstrates the notion of mean.

By using characteristic kernels~\citep{Fukumizu07}, it was proved that no information about the distribution $P$ would be lost after mapping to $\mu_k(P)$. Precisely, $\lVert \mu_k(P)-\mu_k(Q) \rVert_{\mathcal{H}_k}=0$ is equivalent to $P=Q $. This property makes KME a theoretically sound method to represent a distribution. An example of characteristic kernel is the Gaussian kernel
 \begin{equation} \label{eq:gaussian_kernel}
     k(x,x^\prime)=\exp(-\gamma \lVert x-x^\prime \rVert), \gamma>0.
 \end{equation}

 In learning tasks, we often have no access to the true distribution $P$, and consequently to the true embedding $\mu_k(P)$. Therefore, the common practice is to use examples $X=\{x_n\}_{n=1}^N\sim P^N$, which constructs an empirical distribtuion $P_X$, to approximate \eqref{eq:KME_def}:
\begin{equation}\label{eq:KME_empirical}
    \widehat{\mu_k}(P_X)\coloneqq \frac{1}{N}\sum_{n=1}^{N} k(x_n,\cdot)
\end{equation}

If all functions in $\mathcal{H}_k$ are bounded and examples are i.i.d drawn, the empirical KME $\widehat{\mu_k}(P_X)$ converges to the true KME $\mu_k(P)$ at rate $O(1/\sqrt{N})$, measured by RKHS norm \cite[Theorem 1]{Lopez15}.

\subsection{Reduced Set Construction}
Reduced set methods were first proposed to speed up SVM predictions~\citep{SVM_RS} by reducing the number of support vectors, and soon were found useful in kernel mean embeddings~\citep{TNN99_RS} to handle storage and/or computational difficulties.

The empirical KME $\widehat{\mu_k}(P_X)$ is an approximation of true KME $\mu_k(P)$, requiring all the raw examples. It is known that we can approximate the empirical KME with fewer examples. Reduced set methods find a weighted set of points $R=(\bm{\beta},Z)=\{(\beta_m,z_m) \}_{m=1}^M$ in the input space to minimize the distance measured by RKHS norm
\begin{equation}\label{eq:rs_obj}
   \lVert\widehat{\mu_k}(P_X)-\widehat{\mu_k}(P_R) \rVert_{\mathcal{H}_k}^2=\Big\lVert\sum_{n=1}^{N}\frac{1}{N} k(x_n,\cdot)-\sum_{m=1}^{M} \beta_m k(z_m,\cdot) \Big\rVert_{\mathcal{H}_k}^2
\end{equation}

It is trivial to achieve perfect approximation if we are allowed to have the same number of points in the reduced set ($M=N$). Therefore we focus on the $M<N$ case by introducing additional freedom on real-value coefficients $\bm{\beta}$ and vectors $Z$. If points in $Z$ are selected from $X$, it is called reduced set selection. Otherwise, if $\{z_m\}$ are newly constructed vectors, it is called reduced set construction~\citep{ICCV_rs}. Since the latter does not expose raw examples, we apply reduced set construction to compute the specification in the upload phase of our proposal.

\subsection{Kernel Herding}
Kernel herding algorithm is an infinite memory deterministic process that learns to approximate a distribution with a collection of examples~\citep{kernel_herding}. Suppose we want to draw examples $X=\{x_n\}_{n=1}^N$ from distribution $P$, but the probability distribution function of $P$ is unknown. Given the kernel mean embedding $\mu_k(P)$ of $P$, assume $k(x,x)$ is bounded for all $x\in \X$ and the further restrictions of finite-dimensional discrete state spaces~\citep{Welling09}, kernel herding will iteratively draw an example in terms of greedily reducing the following error at every iteration:
\begin{equation} \label{eq:herding_obj}
  \lVert\widehat{\mu_k}(P_X)-\mu_k(P)\rVert^2_{\H_k}=\Big\lVert \frac{1}{N}\sum_{n=1}^N k(x_n,\cdot) -\mu_k(P)\Big\rVert^2_{\H_k}
\end{equation}

A remarkable result of kernel herding is, it decreases the square error in~\eqref{eq:herding_obj} at a rate $O(1/N)$, which is faster than generating independent identically distributed random samples from $P$ at a rate $O(1/\sqrt{N})$.

Comparing with~\eqref{eq:rs_obj} in last section, if we set $\mu_k(P)=\widehat{\mu_k}(P_R)$, kernel herding looks like an ``inverse'' operation of reduced set construction. Reduce set construction in~\eqref{eq:rs_obj} is ``compressing'' the KME, while kernel herding in~\eqref{eq:herding_obj} is ``decompressing'' the information in reduced KME if $N$ is large. We will apply kernel herding in the deployment phase to help recover the information in reduced KMEs.

\section{Framework}
In this section, we first formalize our problem setting with minimal notations, then show how to construct RKME in the upload phase and how to use RKME in the deployment phase.
\subsection{Problem Formulation}\label{sec:formulation}
Suppose there are in total $c$ providers in the upload phase, they build learnwares on their own tasks and generously upload them to a pool for future users. Each of them has a private local dataset $S_i=\{(x_n,y_n)\}_{n=1}^{N_i}$, which reflects the task $T_i$. Task $T_i$ is a pair $(P_i,f)$ defined by a distribution $P_i$ on input space $\X$ and a global optimal rule function $f:\X \rightarrow \Y$,
\begin{equation}
\begin{split}
     \forall i \in [c], \forall (x,y)\in S_i,f(x)=y.
\end{split} 
\end{equation}

All providers are competent, and the local datasets are sufficient to solve their tasks. Formally speaking, their models $\widehat{f}_i$ enjoy a small error rate $\epsilon>0$ with respect to a certain loss function $L$ on their task distribution $P_i$:
\begin{equation} \label{eq:small-loss}
     \forall i \in [c], \mathcal{L}(P_i,f,\widehat{f}_i)=\mathds{E}_{x\sim P_i }\big[L(\widehat{f}_i(x),f(x))\big]\leq \epsilon.
\end{equation} 

With a slight abuse of notation, here $L:\mathcal{Y}\times \mathcal{Y}\rightarrow \mathbb{R}^+$ can be either a regression loss or classification loss. Since tasks $\{T_i\}_{i=1}^c$ are equipped with low-error pre-trained models, they are referred to as solved tasks throughout this paper.

In the deployment phase, a new user wants to solve her current task $T_t$ with only unlabeled testing data $x\sim P_t$. Thus her mission is to learn a good model $\widehat{f}_t$ which minimizes $\mathcal{L}(P_t,f,\widehat{f}_t)$, utilizing the information contained in pre-trained models $\{\widehat{f}_i\}_{i=1}^c$.

This problem seems easy at the first glance. A naive reasoning is:  since all the solved tasks share the same rule function $f$, and each $\widehat{f}_i$ is a low-error estimate of $f$, any of them is a good candidate for $\widehat{f}_t$. However, this is not the case because no assumption between $P_i$ and $P_t$ has been made so far. In an extreme case, the support of $P_t$ may not be covered by the union support of $P_i$, therefore there exist areas where all $\widehat{f}_i$'s can fail.

Put it in a concrete example. The global rule function is a 4-class classifier $f:\X \rightarrow \{a,b,c,d\}$. There are two providers equipped with very ``unlucky'' distribution. One's local dataset only contains points with two classes $\{a,b\}$, and another only sees points labeled $\{b,c\}$. They learn zero-error local classifiers $\{\widehat{f}_1,\widehat{f}_2\}$, which are perfect for their own task and uploaded to the public pool. Then in the deployment phase facing current task $T_t$, suppose all points drawn from $P_t$ are actually labeled class $d$ according to $f$. In this unfortunate case, both pre-trained models $\{\widehat{f}_1,\widehat{f}_2\}$ suffer from 100\% error on $T_t$.

The above example demonstrates that it is difficult to have a low-risk model $\widehat{f}_t$ on the current task without making any assumptions on $P_t$. To this end, we propose two realistic assumptions to model relationships between the current and solved tasks.

\textbf{Task-recurrent assumption}: The first type of assumption is that the distribution of the current task matches one of the solved tasks. The current task $T_t$ is said to be task-recurrent from the solved tasks $\{T_i\}_{i=1}^c$ if there exists $i\in [c]$, such that $P_t = P_i$. 

\textbf{Instance-recurrent assumption}: The second type of assumption is that the distribution of the current task is a convex mixture of solved tasks, i.e. $P_t = \sum_{i=1}^c w_i P_i$, where $\bm{w}=(w_1,\cdots,w_c)\in \Delta^c$ lies in a unit simplex. 

The second assumption is weaker as task-recurrent is a special case for instance-recurrent by setting $\bm{w}$ at a vertex of the unit simplex. However, if we are told that the first assumption holds a priori, it is expected to achieve better performance on the current task.

\subsection{Upload Phase}
In this section, we describe how to compute the reduced KME specification to summarize provider $i$'s local dataset $S_i$ in the upload phase. To lighten the notations, we focus on one provider and temporarily drop the subscript $i$. 

Given a local dataset $S=\{(x_n,y_n)\}_{n=1}^N$, where $x_n \sim P$. Now we use empricial KME to map the empirical distribution defined by $X=\{x_n\}_{n=1}^N$ with a valid kernel function $k$. The empirical KME is $\widehat{\mu_k}(P_X)$ as defined in \eqref{eq:KME_empirical}.

Then our mission is to find the reduced set minimizing \eqref{eq:rs_obj}. Denote $\bm{\beta}=(\beta_1,\cdots,\beta_M)$ and $Z=\{z_1,\cdots,z_M\}$, expanding \eqref{eq:rs_obj} gives
\begin{equation}\label{eq:rs_obj_expand}
     F(\bm{\beta},Z)=\sum_{n,m=1}^N \frac{1}{N^2} k(x_n,x_m)
     +\sum_{n,m=1}^M \beta_n \beta_m k(z_n,z_m) -2 \sum_{n=1}^N \sum_{m=1}^M \frac{\beta_m}{N} \ k(x_n,z_m).
\end{equation}
We adopt the alternating optimization to minimize \eqref{eq:rs_obj_expand}.

\textbf{Fix $Z$ update $\bm{\beta}$.} Suppose vectors in $Z$ are fixed, setting $\frac{\partial F(\bm{\beta},Z)}{\partial \bm{\beta}}=0$ obtains the closed-form solution of $\bm{\beta}$:
\begin{equation} \label{eq:rebeta}
\bm{\beta}=K^{-1}C,
\end{equation}
where
\begin{equation*}
K \in \mathbb{R}^{N\times M}, K_{nm}=k(z_n,z_m), C \in \mathbb{R}^{N\times 1}, C_n=\frac{1}{N}\sum_{m=1}^M k(z_n,x_m).
\end{equation*}

\textbf{Fix $\bm{\beta}$ update $Z$.} When $\bm{\beta}$ is fixed, $\{z_1,\cdots,z_M\}$ in $Z$ are independent in \eqref{eq:rs_obj_expand}, therefore we can iteratively run gradient descent on each $z_m$ as
\begin{equation} \label{eq:re_z}
    z_m^{(t)}=z_m^{(t-1)}-\eta  \frac{\partial F(\bm{\beta},Z)}{\partial z_m}.
\end{equation}

The optimization is summarized in Algorithm~\ref{alg:rs}. If the step size $\eta$ is small, the objective value $F(\bm{\beta},Z)$ will decrease monotonically at both steps, and finally converges.

After the optimization, each provider uploads her model $\widehat{f}_i$, paired with RKME specification $\Phi_i$ (represented by $\bm{\beta}$ and $Z$), into the learnware pool. Raw data examples are inaccessible after the construction by design. Differential privacy can be further ensured by applying techniques in \cite{Priave_Release}, which is an interesting issue but out of our scope. 

An illustration of this phase is presented in Fig.~\ref{fig:upload_phase}. In this illustration, 3 providers upload pre-trained binary classification models and computed RKMEs into the public learnware pool. They are unaware of each other, and their pre-trained models disagree on many areas. The RKMEs ($\Phi_1,\Phi_2,\Phi_3$) are score functions in the raw feature space (denoted by contours, deeper means higher), and also points in the RKHS (denoted by points in a cloud). There is no optimal way to ensemble these models, but the RKME specifications allow future users to appropriately reuse them in the deployment phase.

\begin{algorithm}[htb]
   \caption{Reduced KME Construction}
   \label{alg:rs}
\begin{algorithmic}[1]
\renewcommand{\algorithmicrequire}{\textbf{input:}}
\renewcommand{\algorithmicensure}{\textbf{output:}}
\REQUIRE~~\\
    Local dataset $X=\{x_n\}_{n=1}^N$, kernel fucntion $k$, size of reduced set $M$, iteration parameter $T$
\ENSURE~~\\
  Reduced KME $\Phi(\cdot)=\sum_{m=1}^{M}\beta_m k(z_m,\cdot)$
\renewcommand{\algorithmicrequire}{\textbf{procedure:}}
\REQUIRE~~\\
\STATE Initialize $z^{(0)}_m$ by running $k$-means on $X$, where $k=M$
\FOR{$t=1:T$}
\STATE Update $\bm{\beta}$ by \eqref{eq:rebeta}
\STATE Update each $z_m^{(t)}$ by \eqref{eq:re_z}
\ENDFOR
\end{algorithmic}
\end{algorithm}
\begin{figure*}[htb]
\centering
\includegraphics[width=\textwidth]{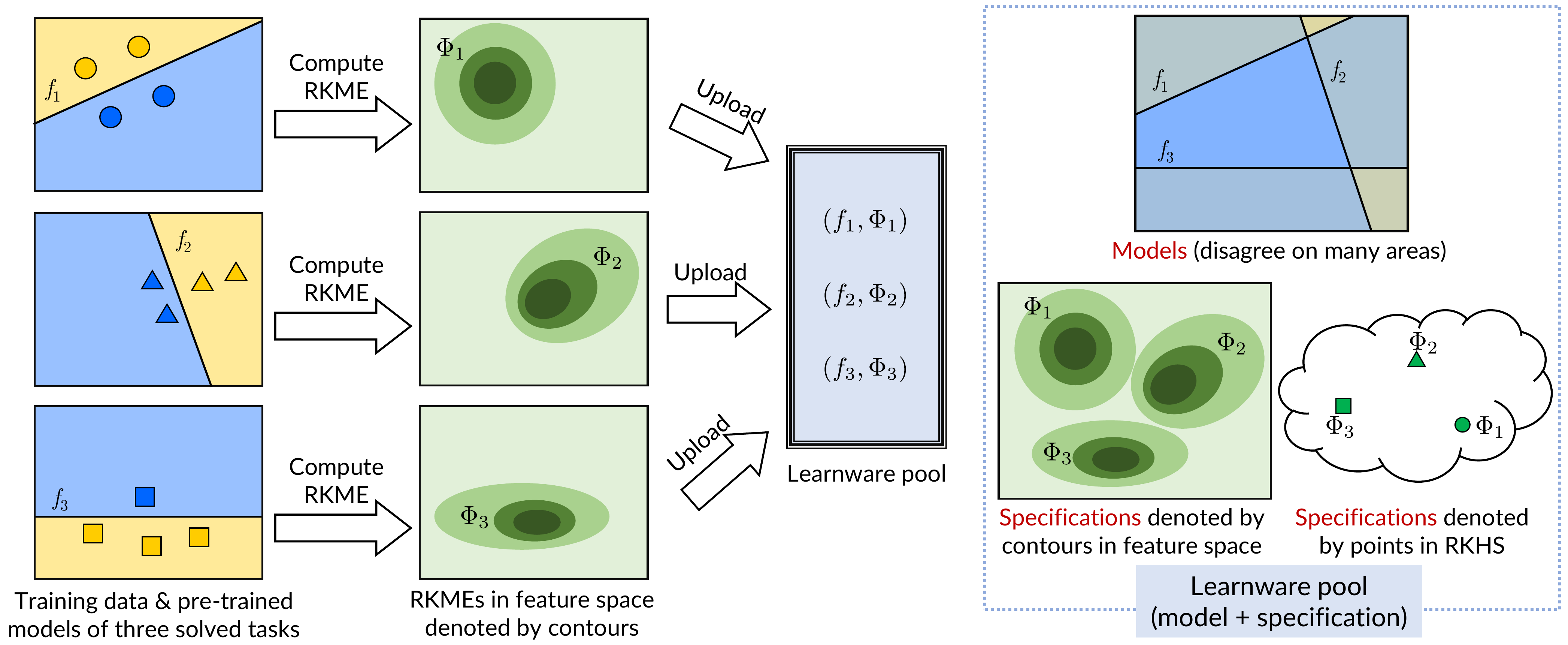}
\caption{An illustration of the upload phase.}\label{fig:upload_phase}
\end{figure*}

\subsection{Deployment Phase}
In this section, we describe how to use RKME to identify useful models in the learnware pool for the current task.  Algorithm~\ref{alg:deploy} shows the overall deployment procedure. As we mentioned in Section~\ref{sec:formulation}, the procedure treats two different recurrent assumptions separately. 

\begin{algorithm}[tb]
   \caption{Deployment Procedure}
   \label{alg:deploy}
\begin{algorithmic}[1]
\renewcommand{\algorithmicrequire}{\textbf{input:}}
\renewcommand{\algorithmicensure}{\textbf{output:}}
\REQUIRE~~\\
    Current task test data $X=\{x_n\}_{n=1}^{N}$, a pool of pre-trained models $\{\widehat{f}_i\}_{i=1}^c$, RKMEs $\{\Phi_i\}_{i=1}^c$
\ENSURE~~\\
  Prediction $Y=\{y_n\}_{n=1}^{N}$
\renewcommand{\algorithmicrequire}{\textbf{procedure:}}
\REQUIRE~~\\
\IF{task-recurrent assumption} 
\STATE $\Phi_t=\sum_{n=1}^{N}\frac{1}{N} k(x_n,\cdot)$
\STATE $i^*=\argmin_i \big\lVert \Phi_t-\Phi_i \big\rVert_{\mathcal{H}_k}^2$ \label{deploy:min_dist}
\STATE $Y=\widehat{f}_{i^*}(X)$ \label{deploy:task_re_predict}
\ENDIF
\IF{instance-recurrent assumption}
\STATE Estimate $\widehat{\bm{w}}$ as \eqref{eq:w_solution} \label{deploy:ira_start}
\STATE Initialize the mimic sample set $S=\emptyset$ \label{deploy:sampling_start}
\WHILE{$|S|$ is not big enough}
\STATE Sample a provider index $i$ by weight $\widehat{w}_i$ \label{deploy:pick_provider}
\STATE Sample an example $x$ by kernel herding as~\eqref{eq:herding_next} \label{deploy:herding}
\STATE $S=S\cup \{(x,i)\}$
\ENDWHILE \label{deploy:sampling_end}
\STATE Train a selector $g$ on mimic sample $S$
\FOR{$n=1:N$}
\STATE $i^*=g(x_n)$\label{deploy:phi_x_start} 
\STATE $y_n=\widehat{f}_{i^*}(x_n)$ \label{deploy:phi_x_end}
\ENDFOR \label{deploy:ira_end}
\ENDIF 
\end{algorithmic}
\end{algorithm}
\subsubsection{Task-recurrent assumption}
When the task-recurrent assumption holds, which means the current distribution matches one of the distributions solved before, our goal is to find out which one fits the best. In Line \ref{deploy:min_dist} of Algorithm \ref{alg:deploy}, we measure the RKHS distance between the testing mean embedding and reduced embeddings in the pool, and figure out the model which was trained on the closest data distribution. Then in Line \ref{deploy:task_re_predict}, we apply the matching model $\widehat{f}_{i^*}$ to predict all the points.

\subsubsection{Overview of instance-recurrent assumption}
When the instance-recurrent assumption holds, which means no single pre-trained model can handle all the testing points, our goal is to determine which model is the most suitable for each instance. The general idea is that we ``mimic'' the test distribution by weighting existing distributions first, then ``recover'' enough data points and learn a model selector on them. Finally, the selector predicts the suitable model $\widehat{f}_i$ for each testing point.

\subsubsection{Estimate mixture weights}
Let us see how to estimate the mixture weights first. By instance-recurrent assumption, we have $P_t = \sum_{i=1}^c w_i P_i$, which implies 
\begin{equation}\label{eq:mixture_model}
\mu_k(P_t) = \sum\nolimits_{i=1}^c w_i \mu_k(P_i).
\end{equation}
Let $\{x_n\}_{n=1}^N\sim P_t$ be the examples from the current task. 
To estimate the weights $\widehat{\bm{w}}$, we aim to minimize:
\begin{equation*}\label{eq:mixture-dist-loss}
\min_{\bm{w}}\Big \lVert \frac{1}{N}\sum_{n=1}^N k(x_n,\cdot)- \sum_{i=1}^c w_i \Phi_i(\cdot) \Big \rVert^2_{\H_k},
\end{equation*}
which is similar to \eqref{eq:rs_obj}, thus the solution $\widehat{\bm{w}}$ is similar to \eqref{eq:rebeta}:
\begin{equation}\label{eq:w_solution}
\widehat{\bm{w}} = H^{-1}C,
\end{equation}
where
\begin{equation*}
  H\in \mathbb{R}^{c\times c},H_{ij}=\langle \Phi_i, \Phi_j\rangle, C\in \mathbb{R}^{c\times 1},C_i=\frac{1}{N}\sum_{n=1}^N \Phi_i(x_n).
\end{equation*}

The $\widehat{\bm{w}}$ measures the weight of each provider's distribution. Given the weights, we are able to unbiasedly pick a provider $i$ in Line~\ref{deploy:pick_provider} of Algorithm~\ref{alg:deploy}, which is the first step to mimic the testing distribution.

\subsubsection{Sampling from RKME}
This subsection explains how to implement Line~\ref{deploy:herding} of Algorithm~\ref{alg:deploy}, i.e., sample examples from the distribution $P_i$ with the help of RKME $\Phi_i$, by applying kernel herding techniques~\citep{kernel_herding}. 

For ease of understanding, we temporarily drop the subscript $i$, slightly abuse the notation $t$ as the iteration number here, and rewrite the iterative herding process in~\cite{herding_thesis} via our notations:
\begin{equation}\label{eq:herding_next}
x_{T+1}=\begin{cases}
\argmax_{x\in \X} \Phi(x), &\text{if $T=0$} \\
\argmax_{x\in \X} \Phi(x)-\frac{1}{T+1}\sum_{t=1}^{T}k(x_t,x), &\text{if $T\geq 1$}.
\end{cases}
\end{equation}
where $x_{T+1}$ is the next example we want to sample from $P$ when $\{x_t\}_{t=1}^T$ have been already sampled. And Proposition 4.8 in~\cite{herding_thesis} shows the following error $\mathcal{E}_T$ will decrease as $O(1/T)$:
\begin{equation*}
  \mathcal{E}_T=\Big\lVert \frac{1}{T}\sum_{t=1}^T k(x_t,\cdot) -\Phi\Big\rVert^2_{\H_k}.
\end{equation*}

Therefore, by iteratively sampling as in~\eqref{eq:herding_next}, we will eventually have a set of examples drawn from $P_i$. Combined with unbiased sampling from providers (Line~\ref{deploy:pick_provider} of Algorithm~\ref{alg:deploy}), a labeled sample set $S\sim \sum \widehat{w}_i P_i$ is constructed.
\subsubsection{Final predictions and illustrations}
When all the previous steps are ready, it is quite easy to make the final prediction. The user will train a selector $g:\X\rightarrow \{1,\cdots,c\}$ on $S$ to predict which pre-trained model in the pool should be selected. The selector can be similar to pre-trained models except that its output space is the index of providers. The final prediction for a test instance $x$ will be $\widehat{f}_{i^*}(x)$, where $i^*=g(x)$ is the selected index.
\begin{figure*}[htb]
\centering
\includegraphics[width=\textwidth]{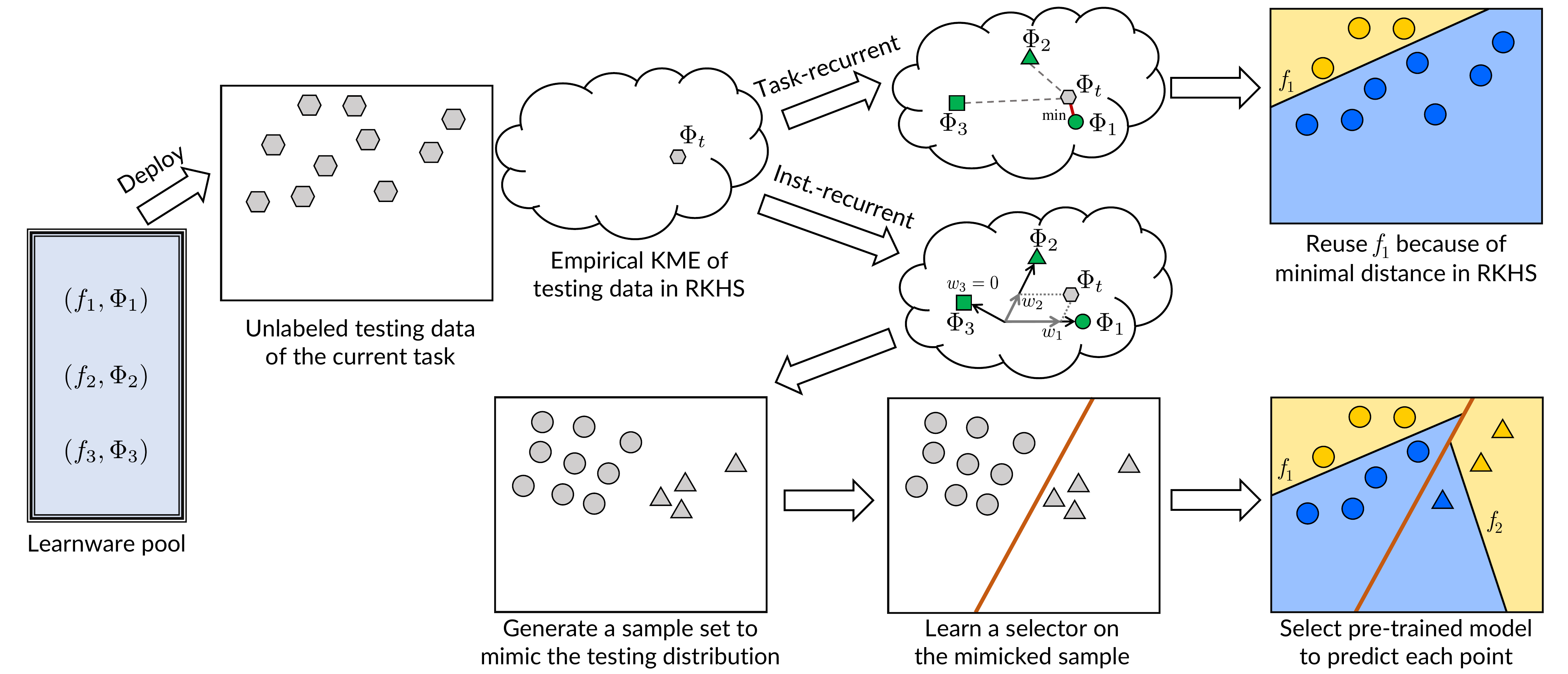}
\caption{An illustration of the deployment phase.}\label{fig:deployment_phase}
\end{figure*}

An illustration of the deployment phase including both task-recurrent and instance-recurrent assumptions is presented in Fig.~\ref{fig:deployment_phase}. It is easier to see the differences between assumptions in the RKHS plotted as a cloud. If task-recurrent, we are finding the closest RKME (which is $\Phi_1$ in that cloud) and only one pre-trained model will be used. If instance-recurrent, we are finding a combination of RKMEs (which is $\Phi_t=w_1 \Phi_1+w_2 \Phi_2$ in that cloud). Each $\Phi_i$ is like a basis in the reproducing kernel Hilbert space, and instance-recurrent assumption is actually saying that $\Phi_t$ can be decomposed by these bases. In that example, because $w_1>w_2$, more circles are generated than triangles in the mimicked sample. There is no square because $w_3=0$. The learned selector shows we should use $f_1$ in the left half and $f_2$ in the right half.

Since we are reusing pre-trained models without modifying them on the current task, our framework accepts any types of model from providers. They can be deep neural classifiers, SVMs with different kernel functions, or gradient boosted tree regressors. As long as the input spaces are identical, these pre-trained models can even have different output spaces.

\section{Theory}
In this section, theoretical results are presented to rigorously justify the reusability of pre-trained models by using RKME specifications via our proposed way, either in task-recurrent assumption or in instance-recurrent assumption.

\subsection{Task-recurrent}
Here we introduce useful propositions regarding MMD and then show the guarantees for the task-recurrent assumption. For simplicity, we omit the subscript $k$ in $\mu_k$ and represent $\mu_k(P)$ as $\mu_P$.

\begin{Proposition}[Upper bound for empirical MMD]\label{prop:mmd-bound}
\textnormal{\cite[Theorem 7]{gretton2012kernel}} Let $P_i$ and $P_t$ be Borel probability measure defined in topological space $\X$ and the associated kernel function is bounded: $0\leq k(x,x')\leq K, \forall x,x'\in \X$. If $P_i = P_t$, then, with probability at least $1-2\delta$, the (biased) empirical \textnormal{MMD} in \eqref{eq:rs_obj} is bounded by:
\[
\frac{1}{2}\textnormal{MMD}_b(\widehat{P_i}, \widehat{P_t}) < \sqrt{\frac{K}{m}}+ \sqrt{\frac{K}{n}} +\sqrt{\frac{K(m+n)\log \frac{1}{\delta}}{2mn}}
\]
for arbitrary small $\delta>0$. And this bound converges to zero when $m,n\to \infty$.
\end{Proposition}

We then know that when task-recurrent assumption is satisfied, the empirical MMD is bounded by a small value. Using such an idea, we further bound the empirical estimation of the task-recurrent case.

\begin{Theorem}[Task-recurrent bound]\label{thm:task-recurrent} Assume $T_t$ is task-recurrent, i.e. $\exists i, P_t = P_i$. The learned model $\widehat{f}_i$ from each solved task satisfy 
\eqref{eq:small-loss}.
Assume the loss function $L(\widehat{f}_i(x),f(x)) \in \H_k$, and upper bounded $|L(\widehat{f}_i(x),f(x))|<K$ $\forall i$. The empirical \textnormal{MMD} between distribution $P_i$ and current distribution $P_t$ can be estimated from
\[
\widehat{\textnormal{MMD}}_b^{(i)} = \Big\lVert\sum_{n=1}^{N}\frac{1}{N} k(x_n,\cdot)-\Phi_i(\cdot) \Big\rVert_{\mathcal{H}_k}^2    
\]
where $x_n\sim P_t$.
Further assume that the minimum empirical \textnormal{MMD}  is bounded above: 
$\min_i \widehat{\textnormal{MMD}}_b^{(i)} \leq \eta$. 
The selected model for current task is
$\widehat{f}_t = \widehat{f}_{i^*}$ s.t.
\[
i^{\ast}=\argmin_i \widehat{\textnormal{MMD}}_b^{(i)}.
\]
The finite sample loss satisfies:
\begin{equation}\label{eq:task-recurrent-bound}
   \widehat{\mathcal{L}}(P_t,f,\widehat{f}_t)=\sum_{x_n\sim P_t }\big[L(\widehat{f}_t(x_n),f(x_n))\big]\leq \epsilon + K\eta + O(m,n)
\end{equation}
where $O(m,n)=O(1/\sqrt{n}+1/\sqrt{m}) \to 0$ as $m,n \to \infty$. 
\end{Theorem}

\begin{spacelessprf}
When task-recurrent assumption holds, our method either correctly identifies the recurrent pre-trained model or not. Let $j$ be the correct recurrent index. If $i^{\ast} = j$,  the result follows directly from \eqref{eq:small-loss}. If $i^{\ast} = i \neq j$, the loss function is
$$
L(\widehat{f}_i(x), f(x)) = L_i(x) = \langle L_i, k(x,\cdot) \rangle.
$$
We then can represent the error of each model in the form of KME. For the current task $T_t$:
$$\mathds{E}_{x\sim P_t } \big[L_t(x)\big] = \langle L_t, \mu_{P_t}\rangle = \langle L_i, \mu_{P_j}\rangle $$
since $P_t = P_j$ and the selected model $\widehat{f}_t = \widehat{f}_i$. However, the correct matching should be $\widehat{f}_t = \widehat{f}_j$. Hence, we are applying model $\widehat{f}_i$ on distribution $P_j$. The empirical loss in $T_t$ is:
$$
|\langle L_i, \widehat{\mu}_{P_j}\rangle| \leq 
\underbrace{|\langle L_i, \widehat{\mu}_{P_j}\rangle - \langle L_i, \widehat{\mu}_{P_i}\rangle|}_{(A)} + \underbrace{|\langle L_i, \widehat{\mu}_{P_i}\rangle|}_{(B)}
$$
We then bound $(A)$ and $(B)$ separately. 
$$(B)\leq |\langle L_i, \widehat{\mu}_{P_i}\rangle|+ |\langle L_i, \widehat{\mu}_{P_i} - \widehat{\mu}_{P_i}\rangle|\leq \epsilon + O(1/\sqrt{n})$$ 
by \eqref{eq:small-loss} and convergence rate for empirical \textnormal{MMD}. 
$$(A) = |\langle L_i, \widehat{\mu}_{P_j} - \widehat{\mu}_{P_i}\rangle| \leq \|L_i\|\|\widehat{\mu}_{P_j} - \widehat{\mu}_{P_i} \|\leq K \eta + O(m,n)$$
To see the bound for empirical embeddings, $$\|\widehat{\mu}_{P_j} - \widehat{\mu}_{P_i} \| \leq  \underbrace{\|\widehat{\mu}_{P_t} - \widehat{\mu}_{P_i}}_{\leq \eta} \| + \underbrace{\|\widehat{\mu}_{P_j} - \widehat{\mu}_{P_t} \| }_{=O(m,n)}$$
the first term is by assumption and the second term is because $P_t=P_i$:
\begin{equation*}
\|\widehat{\mu}_{P_j} - \widehat{\mu}_{P_t}\| \leq\|\widehat{\mu}_{P_j} - \mu_{P_j} \|+\|\widehat{\mu}_{P_t} - \mu_{P_t} \| = O\Big(\frac{1}{\sqrt{m}}+\frac{1}{\sqrt{n}}\Big),
\end{equation*}
which completes the proof.
\end{spacelessprf}
The theorem shows that, in the task-recurrent setting, the test error from our procedure is bounded by the solved task error $\epsilon$ and the approximation error $\eta$ of empirical KME. 
\subsection{Instance-recurrent}
In this subsection, we present the guarantee of our proposal when instance-recurrent assumption $P_t=\sum w_i P_i$ holds. Our analysis is based on the following four steps:
\begin{enumerate}
  \item The estimation of $\{\widehat{w}_i\}$ is close to the true mixture weight $\{w_i\}$.
  \item The quality of the mimicked sample set generated by kernel herding is good.
  \item The error of the learned selector $g$ is bounded.
  \item The error of our estimated predictor $\widehat{f}_t$ is bounded.
\end{enumerate}
\begin{Lemma}[$w$-estimation bound]\label{lem:w_estimation}
Consider the weights estimation procedure stated in~\eqref{eq:w_solution}: $\widehat{\bm{w}} = H^{-1}C,$
where
\begin{equation*}
  H\in \mathbb{R}^{c\times c},H_{ij}=\langle \Phi_i, \Phi_j\rangle, C\in \mathbb{R}^{c\times 1},C_i=\frac{1}{N}\sum_{n=1}^N \Phi_i(x_n),
\end{equation*}
and consider the population embedding $\mu_k(P_t)$ for $\widetilde{\bm{w}} = H^{-1}\widetilde{C}$
where
$$\widetilde{C}\in \mathbb{R}^{c\times 1},\widetilde{C}_i=\langle \Phi_i, \mu_k(P_t) \rangle.$$
Assume $H$ is non-degenerate, i.e. the smallest eigenvalue is bounded below by a positive real number $\lambda$. Then,
$$\|\mu_k(P_t) - \sum_{i=1}^c \widehat{w}_i\Phi_i(\cdot) \| = O(\frac{1}{\sqrt{M}} + \frac{1}{\sqrt{N}})$$
\end{Lemma}

\begin{proof}
The proof starts from the instance-recurrent assumption where $\mu_k(P_t) = \sum_i w_i \mu_k(P_i) $. Rewrite 
\begin{align}
\|\mu_k(P_t) - \sum_{i=1}^c \widehat{w}_i\Phi_i \| 
\leq& \|\sum_i w_i \mu_k(P_i) - \sum_i \widetilde{w}_i \Phi_i \| \label{eq:term1} \\
&+\|\sum_i \widetilde{w}_i \Phi_i - \sum_i \widehat{w}_i \Phi_i\| \label{eq:term2} \\
=&O(\frac{1}{\sqrt{M}} + \frac{1}{\sqrt{N}})
\end{align}
\eqref{eq:term1} goes to $0$ since $\Phi_i$ is a $O(1/\sqrt{M})$-consistent estimator of $\mu_k(P_i)$.
Let $\widetilde{H}\in \mathbb{R}^{c\times c},\widetilde{H}_{ij}=\langle \Phi_i, \mu_k(P_j)\rangle$. $\|H-\widetilde{H}\|_F\leq c^2 O(\frac{1}{\sqrt{M}})$. As $c \ll M$, we have
$\widetilde{w} = H^{-1}\widetilde{H}w = w + H^{-1}(\widetilde{H} -H)w = w + O(\frac{1}{\sqrt{M}})$. As $\Phi_i \in \H$ is bounded, $\eqref{eq:term1} =O(\frac{1}{\sqrt{M}})$.
$\eqref{eq:term2}=\|\sum_i (\widetilde{w}_i-\widehat{w}_i)\Phi_i\|\leq K\|\widetilde{\bm{w}}-\widehat{\bm{w}}\|_F = K\lambda^{-1}\|C-\widetilde{C}\|_F = K^2\lambda^{-1}\|\mu_k(P_t)- \frac{1}{N}\sum_n k(x_n,\cdot)\| =K^2\lambda^{-1} O(\frac{1}{\sqrt{N}}) $, where $K<\infty$ as bounded of RKHS, $\lambda^{-1}<\infty$ as $H$ is non-degenerate and 
$$\|\mu_k(P_t)- \frac{1}{N}\sum_n k(x_n,\cdot)\|=O(\frac{1}{\sqrt{N}})$$ 
as $\{x_n\} \sim P_t$.
By such construction, we are able to see the difference, in terms of Frobenius norm, of the weight vector learned from true embedding $\widetilde{\bm{w}}$ versus from empirical embedding $\widehat{\bm{w}}$. 
\end{proof}
To better understand the instance-recurrent case, for the mixture model in \eqref{eq:mixture_model}, we perceive the component assignment as a latent variable, $I$. Hence, by setting $P_i(x) = P(x|I=i)$ as the conditional distribution of $x$ given that $x$ comes from the $i$-th component,  we can write the distribution of test data as:
$$P_t(x) = \sum_i w_i P_i(x) = \sum_i P(x|i)\Pi(I=i) = \sum_i P(x, i)$$ 
where $\Pi(I=i)$ is the marginal distribution of component assignment variable $I$. This corresponds to partition weights $w_i$, where $\sum_i w_i = 1$. We call the $P(x,i)$ the joint distribution of random variable pair $(X, I)$.

\begin{Proposition}\label{prop:herding_rate}
\cite[Proposition 4]{kernel_herding}
Let $p$ be the target distribution, $T$ be the number of samples generated from kernel herding, and $\hat{p}$ be the empirical distribution from $T$ samples. For any $f\in \H$, the error $|\E[f]_p - \E[f]_{\hat{p}}| = O(T^{-1})$.   Moreover this condition holds uniformly, that $\sup_{\|f\|\neq 1}|\E[f]_p - \E[f]_{\hat{p}}| = O(T^{-1})$. Thus $\|\mu_p  - \mu{\hat{p}}\|
_{\H}= O(T^{-1})$.
\end{Proposition}

The proof can be found in~\cite{kernel_herding}, utilizing Koksma Hlawka inequality. Proposition~\ref{prop:herding_rate} shows that the convergence rate of empirical embedding from kernel herding is $O(T^{-1})$ which is faster than the convergence rate between empirical mean embedding to the population version from empirical samples. Hence, sampling $P(x,i)$ does not make the convergence rate, from $\sum_i \hat{w_i}\Phi_i(\cdot)$ to $\sum_i w_i \mu(P_i)$ slower.
\newline
\textbf{Learning classifier from samples:} 
With the samples generated from the herding step, we train a selector $g$ via the following loss 
$$
\mathcal{L}_c(\hat{P}(x,i),g) = \int L_c(g(x), i) \d\hat{P}(x,i)
$$
We do not have direct access to $P(x,i)$ as the selector is learned from the generated empirical samples. 
We assume the loss $L_c \in \H$ and thus bounded, and the training loss is small, i.e. $\mathcal{L}_c(g^{\ast}; x,i) \leq \varepsilon$ for $g^{\ast}= \argmin_g{\mathcal{L}_c(P(x,i),g)}
$.

\begin{Lemma}\label{lem:classifier_error}
Let optimal selector $g^{\ast}= \argmin_g{\mathcal{L}_c(P(x,i),g)}$ and $\mathcal{L}_c(P(x,i),g^{\ast}) \leq \varepsilon$; let the estimated embedding from the deployment phase and \eqref{eq:w_solution} be $\sum_i\hat{w_i}\Phi_i(\cdot)$. Then, the population loss using the learned classifier $g$ is
$$
\mathcal{L}_c(P(x,i),g) = \int L_c(g(x), i) \d P(x,i)
\leq \varepsilon + O(\frac{1}{\sqrt{N}} + \frac{1}{\sqrt{M}})$$
\end{Lemma}
\begin{proof}
Let $\hat{P}(x,i)$ be herding samples from $\sum_i \hat{w}_i \Phi_i(\cdot)$. Then
$$
\mathcal{L}_c(P(x,i),g) \leq
\underset{{\leq \varepsilon}}{\underbrace{\mathcal{L}_c(P(x,i),g^{\ast})}}
+ \langle L_c(g^{\ast}) , \mu_k(P) - \mu_k(\hat{P})\rangle.
$$

By Lemma \ref{lem:w_estimation} and Proposition \ref{prop:herding_rate}, we know $\mu_k(P) - \mu_k(\hat{P})$ approaches zero at rate $ O(\frac{1}{\sqrt{N}} + \frac{1}{\sqrt{M}})$ in RKHS norm. As $L_c$ is a bounded function, the second term goes to zero at rate $ O(\frac{1}{\sqrt{N}} + \frac{1}{\sqrt{M}})$, which completes the proof.
\end{proof}
Note that $L_c(g^{\ast})$ is a bounded function that only depends on $g^{\ast}$, as $P$ and $\hat{P}$ both represent the joint distribution of $(x,i)$. Alternatively, for discrete random variable $I$ that is finite, it is equivalent to see the embedding as $\mu_k(P) = \int k(x,\cdot) \otimes h(i,\cdot) \d P(x,i)$ where $h(i,\cdot)$ is the linear kernel w.r.t. $I$. 

\begin{Theorem}[Instance-recurrent bound] 
Assume for the source models, $$\forall i, \mathcal{L}(P_i,f, \hat{f}_i)=\mathds{E}_{x\sim P_i }\big[L(\widehat{f}_i(x),f(x))\big]\leq \epsilon,$$ where $L \in \H$ and bounded; the classification error of trained classifier $g$ is small, i.e. $\mathcal{L}_c(P(x,i),g)\leq \varepsilon$, where $L_c \in \H$ and bounded. We further assume a Lipschitz condition between the loss used for task and the loss used for training classifier $g$: $\|L(f, \hat{f}_i) -L(f, \hat{f}_j)\| \leq \eta \| L_c(i,j) \| $ for some $\eta < \infty$. Then the RKME estimator $\hat{f}_t = f_{g(x)}(x)$ satistifies 
$$\mathcal{L}(P_t,f,\hat{f}_t) \leq \epsilon + \varepsilon + O(\frac{1}{\sqrt{N}} + \frac{1}{\sqrt{M}})$$
\end{Theorem}

\begin{proof}
The samples approximating $P_t$ are generated from estimated mean embedding via herding.  Applying the result in Lemma \ref{lem:w_estimation} and Lemma \ref{lem:classifier_error}, $\widehat{w}_i\Phi_i(\cdot)$ are $\delta$-consistent and assignment error is $(\delta+\varepsilon)$-consistent, for $\delta =O(\frac{1}{\sqrt{N}} + \frac{1}{\sqrt{M}}) $.
\begin{align*}
&\mathcal{L}(P_t,f, \hat{f}_t)=\int_{\X} L(f(x),\hat{f}_t(x)) \dP(x|i)p(i)  \\
&= \sum_i w_i \int_{\X_i} L(f(x),\hat{f}_t(x))\dP(x|i)\\
&=\sum_i w_i \int_{\X_i} L(f(x),\hat{f}_{g(x)}(x)) \dP(x|i) \\
& \leq \sum_i w_i \int_{\X_i} L(f(x),\hat{f}_{i}(x)) \dP(x|i)
+ \sum_i w_i \int_{\X_i, g(x)\neq i} \hspace{-.9cm}L(f(x),\hat{f}_{g(x)}(x)) - L(f(x),\hat{f}_{i}(x)) \dP(x|i)\\
&\leq \sum_i w_i \epsilon + \eta \sum_i w_i \int_{\X_i, g(x)\neq i} \hspace{-0.9cm}L_c(g(x), i)\dP(x|i)\\
&\leq \sum_i w_i \epsilon + \eta \sum_i w_i \int_{\X_i} L_c(g(x), i)\dP(x|i)\\
&= \epsilon + \eta \int_{\X} L_c(g(x), i)\dP(x,i)\\
&\leq \epsilon + \eta \varepsilon + O(\frac{1}{\sqrt{N}} + \frac{1}{\sqrt{M}})
\end{align*}
where $\X_i = \{x: I=i\}$.
The third last inequality holds as we are looking into the set where the component assignment are mis-classified. Using the Lipschitz condition, we bounded the loss using training loss results of $g$.
The second last inequality dropping $g(x)=i$ holds because $L_c$ is non-negative loss function.
By Lemma \ref{lem:classifier_error}, the last inequality holds.
\end{proof}

\section{Related Work}
Domain adaptation~\citep{multisource_suvery} is to solve the problem where the training examples (source domain) and the testing examples (target domain) are from different distributions. A common assumption is the examples in source domain are accessible for learning the target domain, while learnware framework is designed to avoid the access at test time. Domain adaptation with multiple sources~\citep{DA_multiple08,DA_multiple18} is related to our problem setting. Their remarkable theoretical results clearly show that given the direct access to distribution, a weighted combination of models can reach bounded risk at the target domain, when the gap of distributions are bounded by R{\'{e}}nyi divergence~\citep{DA_UAI09}. Compared with the literature, it is the first time that the prediction is made from dynamic selection, which is capable of eliminating useless models and more flexible for various types of model. Furthermore, density estimation is considered difficult in high dimensional space, while ours does not depend on estimated density function but implicitly matching distributions via RKME for model selection.

Data privacy is a common concern in practice. In terms of multiple participants setting like ours, multiparty learning~\citep{PP_Aggregation}, and recently a popular special case called federated learning~\citep{federated,federated_google} are related. Existing approaches for multiparty learning usually assume the local dataset owner follows a predefined communication protocol, and they jointly learn one global model by continuously exchanging information to others or a central party. Despite the success of that paradigm such as Gboard presented in \cite{gboard}, we observe that in many real-world scenarios, local data owners are unable to participate in such an iterative process because of no continuous connection to others or a central party. Our two-phase learnware framework avoids the intensive communication, which is preferable when each data owner has sufficient data to learn her own task.

Model reuse methods aim at reusing pre-trained models to help related learning tasks. In the literature, this is also referred to as ``learning from auxiliary classifiers''~\citep{Auxiliary} or ``hypothesis transfer learning''~\citep{Kuzborskij13,HTL_tranformation}. Generally speaking, there are two ways to reuse existing models so far. One is updating the pre-trained model on the current task, like fine-tuning neural networks. Another is training a new model with the help of existing models like biased regularization~\citep{Tommasi14} or knowledge distillation~\citep{nec45,distillation}. Both ways assume all pre-trained models are useful by prior knowledge, without a specification to describe the reusability of each model. Our framework shows the possibility of selecting suitable models from a pool by their resuabilities, which works well even when existing models are ineffective for the current task.

These previous studies did not touch one of the key challenge of learnware~\citep{learnware}: given a pool of pre-trained models, how to judge whether there are some models that are helpful for the current task, without accessing their training data, and how to select and reuse them if there are. To the best of our knowledge, this paper offers the first solution.

\section{Experiments}
To demonstrate the effectiveness of our proposal, we evaluate it on a toy example, two benchmark datasets, and a real-world project at Huawei Technologies Co., Ltd about communication quality.
\subsection{Toy Example} \label{sec:toy}
In this section, we use a synthetic binary classification example including three providers to demonstrate the procedure of our method. This example recalls the intuitive illustration in Fig.~\ref{fig:upload_phase}\&\ref{fig:deployment_phase}, and we will provide the code in CodeOcean~\citep{codeocean}, a cloud-based computational reproducibility platform, to fully reproduce the results and figures.

Fig.~\ref{fig:toy_local_dataset} shows the problem setting. Circle, triangle, and square points are different local datasets from each provider. Each dataset is drawn from a mixture of two Gaussians. The means of these Gaussians are arranged around a circle denoted by the grey dashed line. Points inside the grey circle are labeled as blue class, and points outside the circle are labeled as yellow class, thus the binary classification problem. We should emphasize that their local datasets are unobservable to others, they are plotted in the same figure just for space-saving. RBF kernel SVMs are used as pre-trained models.
\begin{figure}[h]
\centering
\begin{subfigure}[b]{0.23\textwidth}
\captionsetup{skip=1pt}
\includegraphics[width=\textwidth]{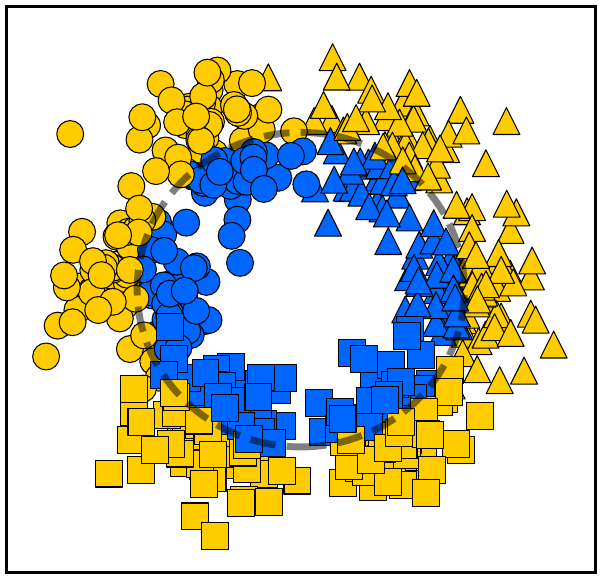}
\caption{Local datasets}\label{fig:toy_local_dataset}
\end{subfigure}
\begin{subfigure}[b]{0.23\textwidth}
\captionsetup{skip=1pt}
\includegraphics[width=\textwidth]{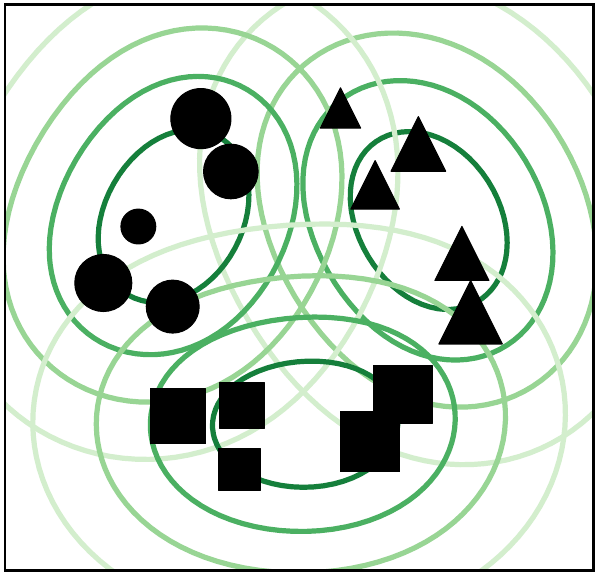}
\caption{RKMEs}\label{fig:toy_rs}
\end{subfigure}
 \caption{Upload phase. (a) Labeled three local private datasets owned by different providers. (b) Constructed points $\{z_m\}$ in the reduced set of KME, bigger marker means larger weight $\{\beta_m\}$. The deeper green contour means higher KME score.}\label{fig:preparation}
\end{figure}

The results of reduced set construction by running Algorithm \ref{alg:rs} are shown in Fig.~\ref{fig:toy_rs}. We set $M=5$ here, which is enough for approximating the empirical KME in this example. Different from the original empirical KME $\sum_{n=1}^N \frac{1}{N}k(x_n,\cdot)$, where all points contribute equally to the embedding, the constructed reduced KME $\sum_{m=1}^M \beta_m k(z_m,\cdot)$ introduced more freedom by using variable weights $\{ \beta_m\}$. In the figure, we use the size of markers to illustrate the value of weights. These reduced sets implicitly ``remember'' the Gaussian mixtures behind local datasets and serve as specifications to tell future users where each pre-trained model works well.

In the deployment phase, we evaluate both task-recurrent and instance-recurrent assumptions. In Fig.~\ref{fig:deployment_tra}, we draw test points from the same distribution of the ``circle'' dataset. As expected, our method successfully finds the match and predict all the data by the pre-trained ``circle'' model. 
\begin{figure}[!htb]
\centering
\begin{subfigure}[b]{0.23\textwidth}
\captionsetup{skip=1pt}
\includegraphics[width=\textwidth]{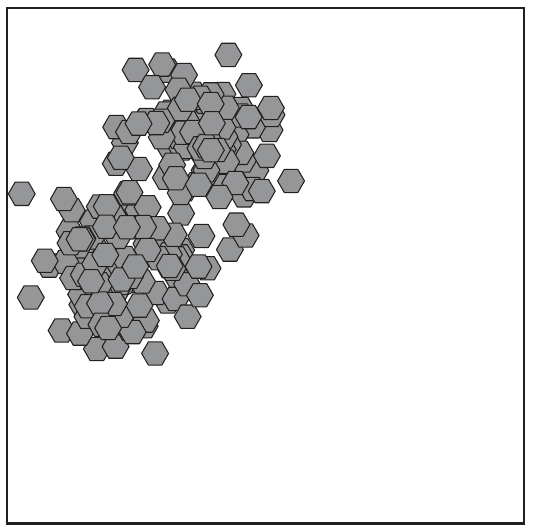}
\caption{Testing data}\label{fig:toy_ira_unlabel}
\end{subfigure}
\begin{subfigure}[b]{0.23\textwidth}
\captionsetup{skip=1pt}
\includegraphics[width=\textwidth]{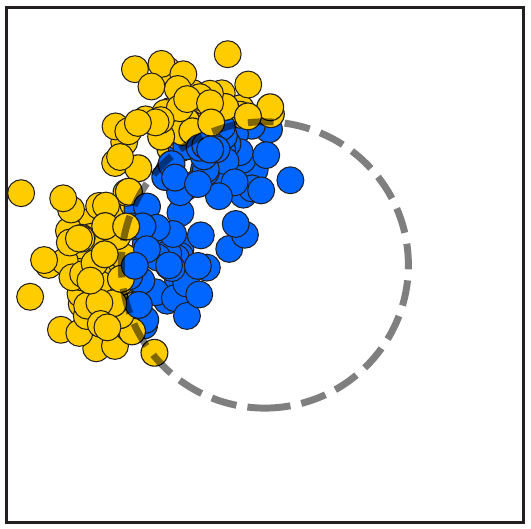}
\caption{Predictions}\label{fig:toy_tra_label}
\end{subfigure}
 \caption{Deployment phase: task-recurrent. (a) Testing data when task-recurrent assumption holds. (b) Predictions, achieved accuracy 97\%.}\label{fig:deployment_tra}
\end{figure}

In instance-recurrent setting, we set the mixture weight of (circle, triangle, square) to $\bm{w}=(0.7,0.3,0.0)$ and test our method. Our estimated mixture weight is $\widehat{\bm{w}}=(0.701,0.285,0.014)$, closing to the groundtruth. Given the accurately estimated mixture weight, we are able to generate a mimicked sample by kernel herding. It is clear in Fig~\ref{fig:toy_ira_mimic} that the drawn distribution is similar to the testing data and with assigned labels. The weight of square is low but not zero, therefore there are still few squares in the sample set. The learned selector divides the feature space into three regions, and all the testing points fall into the  ``circle'' or ``triangle'' region. Predictions in Fig.~\ref{fig:toy_ira_label} achieves accuracy 92.5\%, and errors are mainly made from pre-trained models themselves, not from the selection of ours.
\begin{figure}[!htb]
\centering
\begin{subfigure}[b]{0.23\textwidth}
\captionsetup{skip=1pt}
\includegraphics[width=\textwidth]{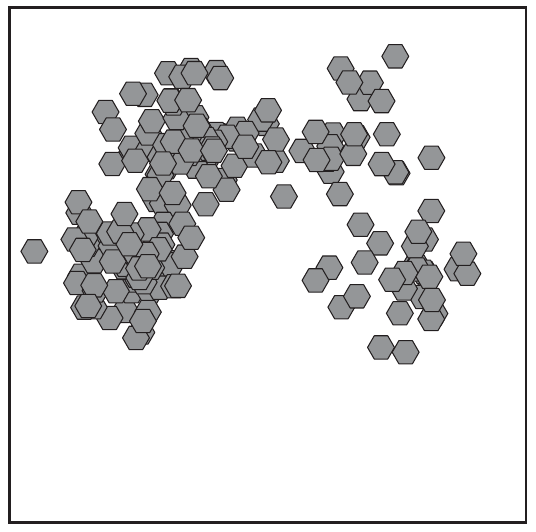}
\caption{Testing data}\label{fig:toy_ira_unlabel}
\end{subfigure}
\begin{subfigure}[b]{0.23\textwidth}
\captionsetup{skip=1pt}
\includegraphics[width=\textwidth]{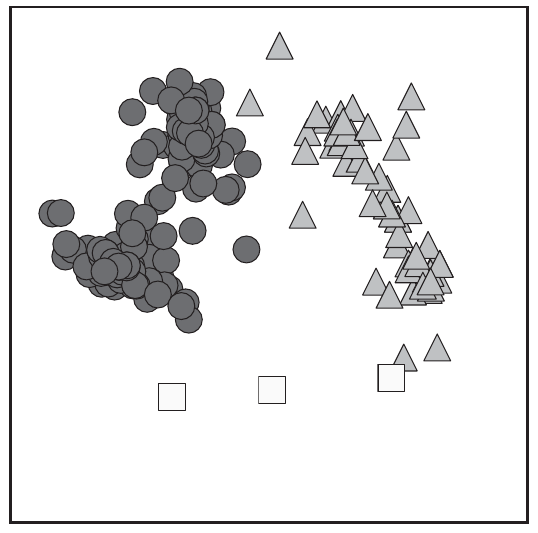}
\caption{Generated data}\label{fig:toy_ira_mimic}
\end{subfigure}
\begin{subfigure}[b]{0.229\textwidth}
\captionsetup{skip=1pt}
\includegraphics[width=\textwidth]{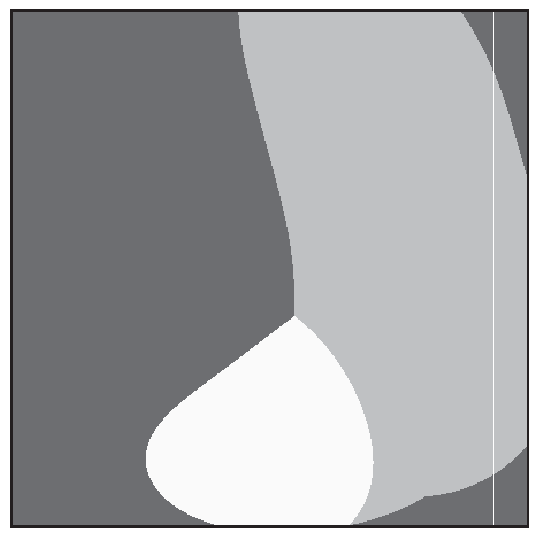}
\caption{Decision of selector}\label{fig:toy_ira_region}
\end{subfigure}
\begin{subfigure}[b]{0.23\textwidth}
\captionsetup{skip=1pt}
\includegraphics[width=\textwidth]{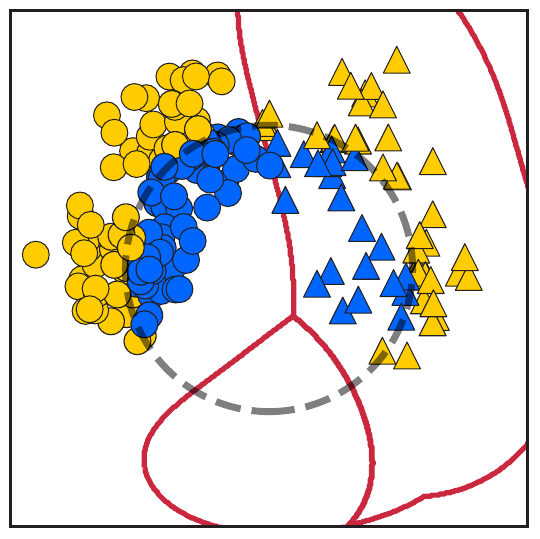}
\caption{Predictions}\label{fig:toy_ira_label}
\end{subfigure}
 \caption{Deployment phase: instance-recurrent. (a) Testing data when instance-recurrent assumption holds. True mixture weight is $\bm{w}=(0.7,0.3,0.0)$. (b) Generated mimicked data by kernel herding with labels. (c) The decision regions of learned selector. (d) Predictions, achieved accuracy 92.5\%. The red line is the decision boundary of the learned selector.}\label{fig:deployment_ira}
\end{figure}

The toy example gives a visual demonstration of our two-phase framework. We can see from this example that the inaccessibility of private training data and reusability of pre-trained models are met. In the next section, we post results on two benchmark datasets.

\subsection{Benchmark} \label{sec:benchmark}
In this section, we evaluate our proposal on two widely used benchmark datasets: image dataset \texttt{CIFAR-100}~\citep{cifar100} and text dataset \texttt{20-newsgroup}~\citep{20newsgroup}. 

\texttt{CIFAR-100} has 100 classes and they are grouped into 20 superclasses, and each superclass contains 5 classes. For example, the superclass ``flower'' includes \{orchid, poppy, rose, sunflower, tulip\}. It is natural to use this dataset to simulate our setting. We divide \texttt{CIFAR-100} into 20 local datasets, each having images from one superclass, and build 5-class local neural network classifiers on them.

\texttt{20-newsgroup} is a popular text classification benchmark and it has similar hierarchical structure as \texttt{CIFAR-100}. There are 5 superclasses \{comp, rec, sci, talk, misc\} and each is considered a local dataset for training local models in the upload phase.

Kernel methods usually cannot work directly on the raw-pixel level or raw-document level, therefore we use off-the-shelf deep models to extract meaningful feature vectors. For \texttt{CIFAR-100}, features are the outputs from the penultimate layer of ResNet-110.\footnote{Trained by running the command of ResNet-110 in \url{https://github.com/bearpaw/pytorch-classification/blob/master/TRAINING.md}} For \texttt{20-newsgroup}, an LSTM is built on GloVe~\citep{glove} word embeddings, and features are extracted from the global max-pooling layer. These feature vectors are used for RKME construction in the upload phase. Gaussian kernel as defined in \eqref{eq:gaussian_kernel} with $\gamma=0.01$ is used in both datasets, and the size of the reduced set is set to $M=10$, which is a tiny ratio of the original datasets.

We compare our method with a naive baseline MAX and a related method HMR~\citep{HMR}. MAX simply uses all the pre-trained models to predict one test instance, and takes out the most confident predicted class. HMR incorporates a communication protocol which exchanges several selected key examples to update models, and then does predictions like MAX. In this comparison we allow HMR to exchange up to 1000 examples. All three methods use the same pool of pre-trained models. Instance-recurrent setting is simulated by randomly mixing testing data from different number of solved tasks. The mean accuracy of 10 times each setting are reported in Table~\ref{table:cifar100}\&\ref{table:newsgroup}, and the last row reports the non-private accuracy of a global model trained on merged data.

\begin{table}
\caption{Results of \texttt{CIFAR-100} in accuracy(\%).}\label{table:cifar100}
\vspace{-0.2cm}
\centering
\begin{tabular}{c c c c c c} 
\toprule
& \multicolumn{1}{c}{Task-recurrent} & \multicolumn{4}{c}{Instance-recurrent}\\
\#Mixing tasks& 1 & 2 & 5 & 10 & 20\\
\midrule
MAX & 43.00 & 42.10 & 41.51 & 41.62 & 41.44\\
HMR & 70.58 & 68.91 & 68.93 & 68.88 & \bftab 68.81\\
Ours& \bftab 86.22 & \bftab 72.91 & \bftab 72.57 & \bftab 71.07 & 68.79\\
\midrule\midrule
Global&75.08& 73.24 & 73.31 & 71.86 & 73.24\\
\bottomrule
\end{tabular}
\end{table}

\begin{table}
\caption{Results of \texttt{20-newsgroup} in accuracy(\%).}\label{table:newsgroup}
\vspace{-0.2cm}
\centering
\begin{tabular}{c c c c c c} 
\toprule
& \multicolumn{1}{c}{Task-recurrent} & \multicolumn{4}{c}{Instance-recurrent}\\
\#Mixing tasks& 1 & 2 & 3 & 4 & 5\\
\midrule
MAX & 58.65 & 55.76 & 53.03 & 51.94 & 50.68\\
HMR & 72.01 & 72.19 & 70.86 & 70.53 & 70.09\\
Ours& \bftab 83.13 & \bftab 76.03 & \bftab 75.10 & \bftab 74.02 & \bftab 72.68\\
\midrule\midrule
Global&72.06& 73.24 & 73.31 & 71.86 & 73.24\\
\bottomrule
\end{tabular}
\end{table}

It is clear that our method performs best with a large margin in the task-recurrent setting. Other methods cannot exploit the prior knowledge that the current task is identical to one of the solved tasks, while our minimum-MMD measure can successfully find out the fittest pre-trained model. 

In the instance-recurrent setting, ours is the best in most cases. We are even better than the non-private global model in the \texttt{20-newsgroup} dataset. It is possible because the global model is an ERM optimizer on the merged data, which is the best model for i.i.d testing examples but not adaptive to a changed unknown distribution. While ours can estimate the mixing weight and adapt to a different biased test distribution in the deployment phase. Ours is increasingly better when the number of mixing tasks goes smaller, because we can preclude some impossible output classes by selecting right pre-trained models.

Besides, we should keep in mind that the comparison is unfair because HMR and global are not fully privacy-preserving methods. Our proposal gets better or competitive performance without exposing any raw data points.

Section \ref{sec:toy} and \ref{sec:benchmark} show results on classification problems. We then apply ours to a real regression problem. 
\subsection{Real-World Project}
Communication quality is the key to user experience for a telecommunications company. We participated in an industrial project called \texttt{crystal-voice} at Huawei Technologies Co., Ltd. Huawei tested a novel technology ``deep cover'' on base stations to improve the quality. But engineers observed the gain of quality varies because of differences about user behaviors and environments among stations. They want to predict how much can we gain in a new base station, to decide whether it is profitable to deploy ``deep cover'' on it.

Every user covered by a base station is represented by a feature vector, and a real-valued quality gain. It is strictly forbidden to move users' information out of stations, but each station has enough data to build a strong local model and share it in a pool. Therefore, our proposal is a wise choice to handle this problem.

In the upload phase, a local ridge regression model is trained in each base station. We then construct RKME (set size $M=50$, Gaussian kernel $\gamma=0.5$) as the specification, and upload the models and specifications into a learnware pool. All the vectors in the specification are constructed ``pseudo'' users, protecting the raw information from thousands of users. 

In the deployment phase, we run instance-recurrent procedure on a new base station. There are 8 anonymous base stations in total, therefore we test our method 8 times. At each time, we select one of them as the current task and the rest 7 as solved tasks. 

Four methods are compared with ours. Two model reuse baselines RAND/AVG and two transfer learning methods KMM~\citep{KMM}/mSDA~\citep{MSDA}. RAND means randomly selecting one pre-trained model from other base stations to predict each user's gain. AVG means averaging the outputs of all regressors in the model pool as predictions. KMM reweights source data to train a better model for the testing data. mSDA learns robust feature representations over all domains.

We should notice that the model reuse methods are private, but transfer learning methods are non-private because they need to see both testing and training data in the deployment phase. The mean results are reported in Table~\ref{table:crystal}. MAX and HMR in Section \ref{sec:benchmark} cannot be used for this regression task, while our framework is agnostic to the task and type of pre-trained models.

\begin{table}[htb]
\caption{Results on regressing quality gain of  the \texttt{crystal-voice} project. ``$\downarrow$'' means the lower the better, ``$\uparrow$'' means the higher the better. ``Model reuse'' methods are private, while ``transfer'' methods are non-private.} \label{table:crystal} 
\vspace{-0.2cm}
\setlength{\tabcolsep}{7.5pt}
\centering
\begin{tabular}{c c c c c c c} 
\toprule
  & & RMSE$\downarrow$ & 3p30$\uparrow$ & 5p30$\uparrow$ & 3f1$\uparrow$ \\
\midrule
 \multirow{3}{*}{Model reuse} 
 & RAND  & .0363 & .3730  & .4412 & .7320 \\
 & AVG   & .0326 & .4272  & .4535 & .7712 \\
 & Ours  & \bftab .0279 & \bftab .5281 & .5205 & \bftab .8082 \\
 \midrule\midrule
 \multirow{2}{*}{Transfer} 
 & KMM   & .0291 & .5018  & .5222 & .7911 \\
 & mSDA  & .0285 & .5105  & \bftab .5324 & .8034 \\
\bottomrule
\end{tabular}
\end{table}
Our method not only outperforms model reuse baselines in terms of root-mean-square error (RMSE), but is also superior on the other measurements required in the real business. ``3p30'' is the ratio of users whose gain value is above 3\% and the prediction error is lower than 30\%. ``5p30'' is defined similarly. ``3f1'' is the F1-measure if we consider the users whose gain value above 3\% are positive class. Our method is even better than mSDA and KMM in some measurements. Considering these two transfer learning methods break the privacy, ours sacrifice a little performance in ``5p30'' while keeping the data safe in base stations.

\section{Conclusion} 
In this paper, we propose reduced kernel mean embedding as the specification in the learnware paradigm, and implement a two-phase pipeline based on it. RKME is shown to protect raw training data in the upload phase and can identify reusable pre-trained models in the deployment phase. Experimental results, including a real industrial project at Huawei, validate the effectiveness of it. This is the first valid specification with practical success to our best knowledge.

In the future, we plan to incorporate more powerful kernel methods to directly measure the similarity in the raw high-dimensional feature space when constructing RKME. It remains an open challenge to design other types of valid specifications, under assumptions which are even weaker than the instance-recurrent assumption.
\bibliography{myRef}

\begin{thebibliography}{38}
\providecommand{\natexlab}[1]{#1}
\providecommand{\url}[1]{\texttt{#1}}
\expandafter\ifx\csname urlstyle\endcsname\relax
  \providecommand{\doi}[1]{doi: #1}\else
  \providecommand{\doi}{doi: \begingroup \urlstyle{rm}\Url}\fi

\bibitem[Arif and Vela(2009)]{ICCV_rs}
O.~Arif and P.~A. Vela.
\newblock Kernel map compression using generalized radial basis functions.
\newblock In \emph{{IEEE} 12th International Conference on Computer Vision},
  pages 1119--1124, 2009.

\bibitem[Balog et~al.(2018)Balog, Tolstikhin, and
  Sch{\"{o}}lkopf]{Priave_Release}
M.~Balog, I.~O. Tolstikhin, and B.~Sch{\"{o}}lkopf.
\newblock Differentially private database release via kernel mean embeddings.
\newblock In \emph{Proceedings of the 35th International Conference on Machine
  Learning}, pages 423--431, 2018.

\bibitem[Burges(1996)]{SVM_RS}
C.~J.~C. Burges.
\newblock Simplified support vector decision rules.
\newblock In \emph{Proceedings of the 13th International Conference on Machine
  Learning}, pages 71--77, 1996.

\bibitem[Chen et~al.(2012)Chen, Xu, Weinberger, and Sha]{MSDA}
M.~Chen, Z.~E. Xu, K.~Q. Weinberger, and F.~Sha.
\newblock Marginalized denoising autoencoders for domain adaptation.
\newblock In \emph{Proceedings of the 29th International Conference on Machine
  Learning}, 2012.

\bibitem[Chen(2013)]{herding_thesis}
Y.~Chen.
\newblock \emph{Herding: Driving Deterministic Dynamics to Learn and Sample
  Probabilistic Models DISSERTATION}.
\newblock PhD thesis, University of California, Irvine, 2013.

\bibitem[Chen et~al.(2010)Chen, Welling, and Smola]{kernel_herding}
Y.~Chen, M.~Welling, and A.~J. Smola.
\newblock Super-samples from kernel herding.
\newblock In \emph{Proceedings of the 26th Conference on Uncertainty in
  Artificial Intelligence}, pages 109--116, 2010.

\bibitem[Clyburne-Sherin et~al.(2019)Clyburne-Sherin, Fei, and
  Green]{codeocean}
A.~Clyburne-Sherin, X.~Fei, and S.~A. Green.
\newblock Computational reproducibility via containers in social psychology.
\newblock \emph{Meta-Psychology}, 3, 2019.

\bibitem[Doran et~al.(2014)Doran, Muandet, Zhang, and
  Sch{\"{o}}lkopf]{KME_casual}
G.~Doran, K.~Muandet, K.~Zhang, and B.~Sch{\"{o}}lkopf.
\newblock A permutation-based kernel conditional independence test.
\newblock In \emph{Proceedings of the 30th Conference on Uncertainty in
  Artificial Intelligence}, pages 132--141, 2014.

\bibitem[Du et~al.(2017)Du, Koushik, Singh, and
  P{\'{o}}czos]{HTL_tranformation}
S.~S. Du, J.~Koushik, A.~Singh, and B.~P{\'{o}}czos.
\newblock Hypothesis transfer learning via transformation functions.
\newblock In \emph{Advances in Neural Information Processing Systems}, pages
  574--584, 2017.

\bibitem[Duan et~al.(2009)Duan, Tsang, Xu, and Chua]{Auxiliary}
L.~Duan, I.~W. Tsang, D.~Xu, and T.-S. Chua.
\newblock Domain adaptation from multiple sources via auxiliary classifiers.
\newblock In \emph{Proceedings of the 26th International Conference on Machine
  Learning}, pages 289--296, 2009.

\bibitem[Fukumizu et~al.(2007)Fukumizu, Gretton, Sun, and
  Sch{\"{o}}lkopf]{Fukumizu07}
K.~Fukumizu, A.~Gretton, X.~Sun, and B.~Sch{\"{o}}lkopf.
\newblock Kernel measures of conditional dependence.
\newblock In \emph{Advances in Neural Information Processing Systems}, pages
  489--496, 2007.

\bibitem[Gretton et~al.(2012)Gretton, Borgwardt, Rasch, Sch{\"o}lkopf, and
  Smola]{gretton2012kernel}
A.~Gretton, K.~M. Borgwardt, M.~J. Rasch, B.~Sch{\"o}lkopf, and A.~Smola.
\newblock A kernel two-sample test.
\newblock \emph{Journal of Machine Learning Research}, 13\penalty0
  (Mar):\penalty0 723--773, 2012.

\bibitem[Hard et~al.(2018)Hard, Rao, Mathews, Beaufays, Augenstein, Eichner,
  Kiddon, and Ramage]{gboard}
A.~Hard, K.~Rao, R.~Mathews, F.~Beaufays, S.~Augenstein, H.~Eichner, C.~Kiddon,
  and D.~Ramage.
\newblock Federated learning for mobile keyboard prediction.
\newblock \emph{arXiv preprint arXiv:1811.03604}, 2018.

\bibitem[Hinton et~al.(2014)Hinton, Vinyals, and Dean]{distillation}
G.~E. Hinton, O.~Vinyals, and J.~Dean.
\newblock Distilling the knowledge in a neural network.
\newblock In \emph{NIPS Workshop on Deep Learning and Representation Learning},
  2014.

\bibitem[Hoffman et~al.(2018)Hoffman, Mohri, and Zhang]{DA_multiple18}
J.~Hoffman, M.~Mohri, and N.~Zhang.
\newblock Algorithms and theory for multiple-source adaptation.
\newblock In \emph{Advances in Neural Information Processing Systems}, pages
  8256--8266, 2018.

\bibitem[Huang et~al.(2006)Huang, Smola, Gretton, Borgwardt, and
  Sch{\"{o}}lkopf]{KMM}
J.~Huang, A.~J. Smola, A.~Gretton, K.~M. Borgwardt, and B.~Sch{\"{o}}lkopf.
\newblock Correcting sample selection bias by unlabeled data.
\newblock In \emph{Advances in Neural Information Processing Systems}, pages
  601--608, 2006.

\bibitem[Jitkrittum et~al.(2016)Jitkrittum, Szab{\'{o}}, Chwialkowski, and
  Gretton]{Interpretable_TSS}
W.~Jitkrittum, Z.~Szab{\'{o}}, K.~P. Chwialkowski, and A.~Gretton.
\newblock Interpretable distribution features with maximum testing power.
\newblock In \emph{Advances in Neural Information Processing Systems}, pages
  181--189, 2016.

\bibitem[Joachims(1997)]{20newsgroup}
T.~Joachims.
\newblock A probabilistic analysis of the rocchio algorithm with {TFIDF} for
  text categorization.
\newblock In \emph{Proceedings of the 14th International Conference on Machine
  Learning}, pages 143--151, 1997.

\bibitem[Konečný et~al.(2016)Konečný, McMahan, Yu, Richtarik, Suresh, and
  Bacon]{federated_google}
J.~Konečný, H.~B. McMahan, F.~X. Yu, P.~Richtarik, A.~T. Suresh, and
  D.~Bacon.
\newblock Federated learning: Strategies for improving communication
  efficiency.
\newblock In \emph{NIPS Workshop on Private Multi-Party Machine Learning},
  2016.

\bibitem[Krizhevsky(2009)]{cifar100}
A.~Krizhevsky.
\newblock Learning multiple layers of features from tiny images.
\newblock Technical report, 2009.

\bibitem[Kuzborskij and Orabona(2013)]{Kuzborskij13}
I.~Kuzborskij and F.~Orabona.
\newblock Stability and hypothesis transfer learning.
\newblock In \emph{Proceedings of the 30th International Conference on Machine
  Learning}, pages 942--950, 2013.

\bibitem[Lopez{-}Paz et~al.(2015)Lopez{-}Paz, Muandet, Sch{\"{o}}lkopf, and
  Tolstikhin]{Lopez15}
D.~Lopez{-}Paz, K.~Muandet, B.~Sch{\"{o}}lkopf, and I.~O. Tolstikhin.
\newblock Towards a learning theory of cause-effect inference.
\newblock In \emph{Proceedings of the 32nd International Conference on Machine
  Learning}, pages 1452--1461, 2015.

\bibitem[Mansour et~al.(2008)Mansour, Mohri, and Rostamizadeh]{DA_multiple08}
Y.~Mansour, M.~Mohri, and A.~Rostamizadeh.
\newblock Domain adaptation with multiple sources.
\newblock In \emph{Advances in Neural Information Processing Systems}, pages
  1041--1048, 2008.

\bibitem[Mansour et~al.(2009)Mansour, Mohri, and Rostamizadeh]{DA_UAI09}
Y.~Mansour, M.~Mohri, and A.~Rostamizadeh.
\newblock Multiple source adaptation and the r{\'{e}}nyi divergence.
\newblock In \emph{Proceedings of the 25th Conference on Uncertainty in
  Artificial Intelligence}, pages 367--374, 2009.

\bibitem[Muandet and Sch{\"{o}}lkopf(2013)]{KME_anomaly}
K.~Muandet and B.~Sch{\"{o}}lkopf.
\newblock One-class support measure machines for group anomaly detection.
\newblock In \emph{Proceedings of the 29th Conference on Uncertainty in
  Artificial Intelligence}, pages 449–--458, 2013.

\bibitem[Muandet et~al.(2017)Muandet, Fukumizu, Sriperumbudur, and
  Sch{\"{o}}lkopf]{KME_survey}
K.~Muandet, K.~Fukumizu, B.~K. Sriperumbudur, and B.~Sch{\"{o}}lkopf.
\newblock Kernel mean embedding of distributions: {A} review and beyond.
\newblock \emph{Foundations and Trends in Machine Learning}, 10\penalty0
  (1-2):\penalty0 1--141, 2017.

\bibitem[Pathak et~al.(2010)Pathak, Rane, and Raj]{PP_Aggregation}
M.~A. Pathak, S.~Rane, and B.~Raj.
\newblock Multiparty differential privacy via aggregation of locally trained
  classifiers.
\newblock In \emph{Advances in Neural Information Processing Systems}, pages
  1876--1884, 2010.

\bibitem[Pennington et~al.(2014)Pennington, Socher, and Manning]{glove}
J.~Pennington, R.~Socher, and C.~D. Manning.
\newblock Glove: Global vectors for word representation.
\newblock In \emph{Proceedings of the 2014 Conference on Empirical Methods in
  Natural Language Processing}, pages 1532--1543, 2014.

\bibitem[Sch{\"{o}}lkopf and Smola(2002)]{Learning_with_Kernels}
B.~Sch{\"{o}}lkopf and A.~J. Smola.
\newblock \emph{Learning with Kernels: support vector machines, regularization,
  optimization, and beyond}.
\newblock {MIT} Press, 2002.

\bibitem[Sch{\"{o}}lkopf et~al.(1999)Sch{\"{o}}lkopf, Mika, Burges, Knirsch,
  M{\"{u}}ller, R{\"{a}}tsch, and Smola]{TNN99_RS}
B.~Sch{\"{o}}lkopf, S.~Mika, C.~J.~C. Burges, P.~Knirsch, K.~M{\"{u}}ller,
  G.~R{\"{a}}tsch, and A.~J. Smola.
\newblock Input space versus feature space in kernel-based methods.
\newblock \emph{{IEEE} Transactions on Neural Networks}, 10\penalty0
  (5):\penalty0 1000--1017, 1999.

\bibitem[Smola et~al.(2007)Smola, Gretton, Song, and
  Sch{\"{o}}lkopf]{Smola07KME}
A.~J. Smola, A.~Gretton, L.~Song, and B.~Sch{\"{o}}lkopf.
\newblock A {H}ilbert space embedding for distributions.
\newblock In \emph{{ALT}}, pages 13--31, 2007.

\bibitem[Sun et~al.(2015)Sun, Shi, and Wu]{multisource_suvery}
S.~Sun, H.~Shi, and Y.~Wu.
\newblock A survey of multi-source domain adaptation.
\newblock \emph{Information Fusion}, 24:\penalty0 84--92, 2015.

\bibitem[Tommasi et~al.(2014)Tommasi, Orabona, and Caputo]{Tommasi14}
T.~Tommasi, F.~Orabona, and B.~Caputo.
\newblock Learning categories from few examples with multi model knowledge
  transfer.
\newblock \emph{{IEEE} Transactions on Pattern Analysis and Machine
  Intelligence}, 36\penalty0 (5):\penalty0 928--941, 2014.

\bibitem[Welling(2009)]{Welling09}
M.~Welling.
\newblock Herding dynamic weights for partially observed random field models.
\newblock In \emph{Proceedings of the 25th Conference on Uncertainty in
  Artificial Intelligence}, pages 599--606, 2009.

\bibitem[Wu et~al.(2019)Wu, Liu, and Zhou]{HMR}
X.-Z. Wu, S.~Liu, and Z.-H. Zhou.
\newblock Heterogeneous model reuse via optimizing multiparty multiclass
  margin.
\newblock In \emph{Proceedings of the 36th International Conference on Machine
  Learning}, pages 6840--6849, 2019.

\bibitem[Yang et~al.(2019)Yang, Liu, Chen, and Tong]{federated}
Q.~Yang, Y.~Liu, T.~Chen, and Y.~Tong.
\newblock Federated machine learning: Concept and applications.
\newblock \emph{ACM Transactions on Intelligent Systems and Technology},
  10\penalty0 (2):\penalty0 12, 2019.

\bibitem[Zhou(2016)]{learnware}
Z.-H. Zhou.
\newblock Learnware: on the future of machine learning.
\newblock \emph{Frontiers of Computer Science}, 10\penalty0 (4):\penalty0
  589--590, 2016.

\bibitem[Zhou and Jiang(2004)]{nec45}
Z.-H. Zhou and Y.~Jiang.
\newblock Nec4.5: Neural ensemble based {C4.5}.
\newblock \emph{{IEEE} Transactions on Knowledge and Data Engineering},
  16\penalty0 (6):\penalty0 770--773, 2004.

\end{thebibliography}
\bibliographystyle{abbrvnat}

\end{document}